\documentclass[conference,compsoc]{IEEEtran}
\IEEEoverridecommandlockouts

\usepackage{cite}
\usepackage{amsmath,amssymb,amsfonts}
\usepackage{algorithmic}
\usepackage{graphicx}
\usepackage{textcomp}
\usepackage{xcolor}
\def\BibTeX{{\rm B\kern-.05em{\sc i\kern-.025em b}\kern-.08em
    T\kern-.1667em\lower.7ex\hbox{E}\kern-.125emX}}

\usepackage{amsmath}
\usepackage{amssymb}
\usepackage{mathtools}
\usepackage{amsthm}

\usepackage[capitalize,noabbrev]{cleveref}

\theoremstyle{plain}
\newtheorem{theorem}{Theorem}[section]

\theoremstyle{definition}
\newtheorem{definition}[theorem]{Definition}

\theoremstyle{remark}

\usepackage{array} 
\usepackage{multirow} 
\usepackage{siunitx} 
\usepackage{arydshln} 
\usepackage{booktabs}
\usepackage{threeparttable}

\begin{document}

\title{Ignoring Directionality Leads to Compromised Graph Neural Network Explanations
}

\author{\IEEEauthorblockN{Changsheng Sun}
\IEEEauthorblockA{
\textit{National University of Singapore}\\
Singapore \\
cssun@u.nus.edu}
\and
\IEEEauthorblockN{Xinke Li}
\IEEEauthorblockA{
\textit{City University of Hong Kong}\\
Hong Kong SAR \\
xinkeli@cityu.edu.hk}
\and
\IEEEauthorblockN{Jin Song Dong}
\IEEEauthorblockA{
\textit{National University of Singapore}\\
Singapore \\
dcsdjs@nus.edu.sg}
}

\maketitle

\begin{abstract}

~Graph Neural Networks (GNNs) are increasingly used in critical domains, where reliable explanations are vital for supporting human decision-making. However, the common practice of \textit{graph symmetrization} discards directional information, leading to significant information loss and misleading explanations.
Our analysis demonstrates how this practice compromises explanation fidelity. Through theoretical and empirical studies, we show that preserving directional semantics significantly improves explanation quality, ensuring more faithful insights for human decision-makers. These findings highlight the need for direction-aware GNN explainability in security-critical applications.
\end{abstract}


\section{Introduction}

Graph Neural Networks (GNNs) have emerged as a powerful tool for modeling relational data in applications such as financial fraud detection \cite{li2024anomaly, weber2019anti} and social network analysis \cite{xu2018powerful}. As GNNs are increasingly deployed in safety-critical domains where their decisions impact human lives and societal well-being \cite{leslie2019understanding, lords2018ai}, ensuring their trustworthiness has become essential.

Unlike traditional software systems, where correctness can often be ensured through formal verification \cite{sun2009pat, clarke1996formal}, deep learning models—including GNNs—function as \textit{black boxes}, making it difficult to validate their decisions. 
To address this, \textit{explainability} has become essential for deploying GNNs in real-world decision-making pipelines. 
Recently, post-hoc explanation methods such as GNNExplainer~\cite{ying2019gnnexplainer} and PGExplainer~\cite{luo2020parameterized} are widely used to \textbf{enhance user trust, facilitate model debugging for developers, and provide external validation for regulatory compliance} in these black-box GNN models.

A useful analogy can be drawn between explaining convolutional neural networks (CNNs) and GNNs. As shown in Figure~\ref{fig:explain_cnn_vs_gnn}, CNN explainability method Grad-CAM~\cite{selvaraju2017grad}, highlight key image regions influencing a prediction—e.g., focusing on a dog's face to classify it as "dog." Similarly, GNN explainers identify critical subgraph structures affecting predictions. For instance, in financial crime detection, a GNN may classify a transaction network as "Money Laundering" by recognizing an illicit transaction loop. While both tasks extract important features, graph explanations must account for relational and structural dependencies.

\begin{figure}[t]
    \begin{center} 
        \includegraphics[width=0.9\linewidth]{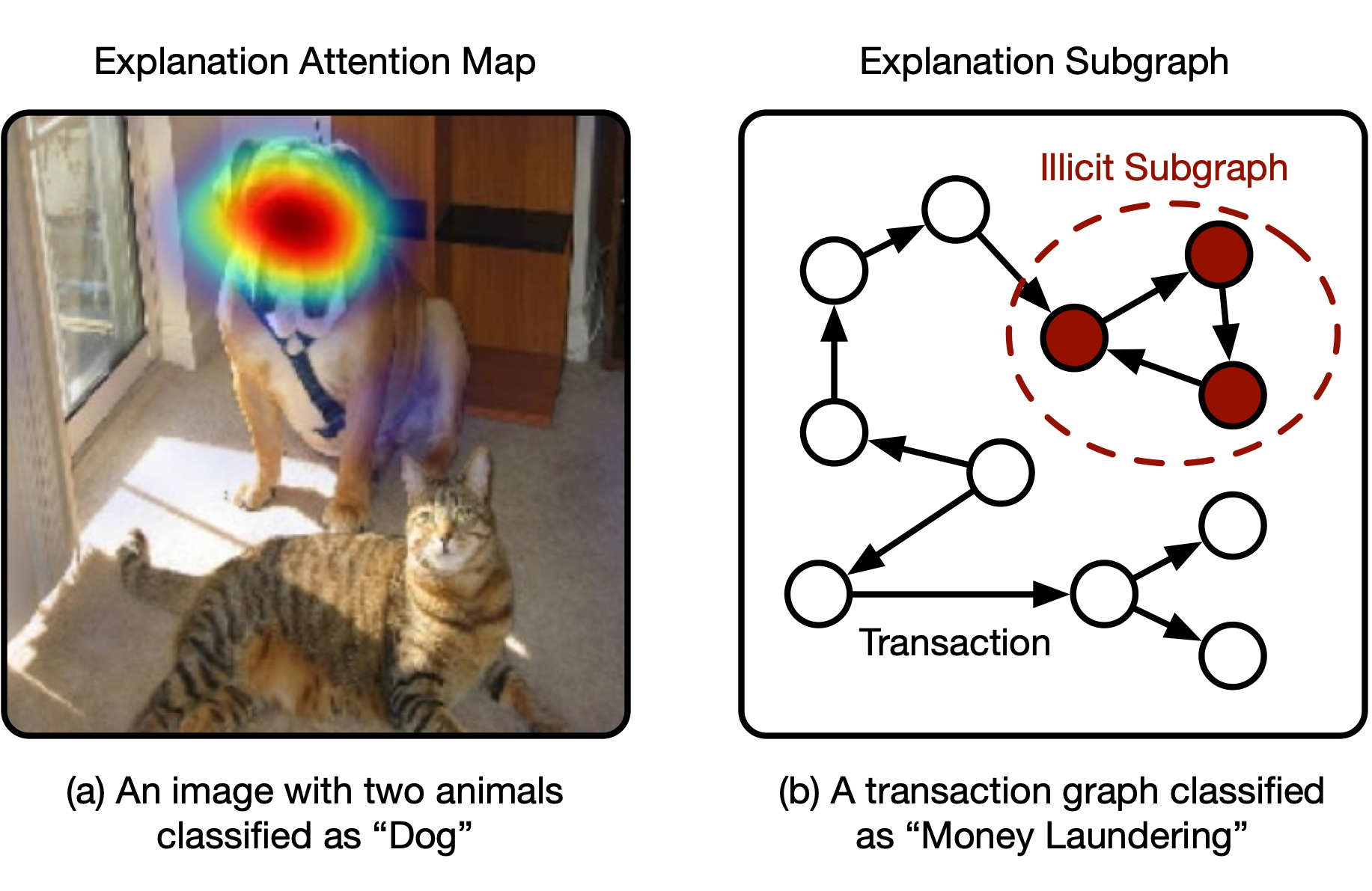}
    \end{center}
    \caption{Comparison of explanation tasks in image classification vs. graph learning. (a) A CNN predicts "dog" by focusing on facial features. (b) A GNN detects "Money Laundering" by recognizing an illicit transaction loop. Note that we need to be able to identify the directional information to use GNN for transaction graph.}
    \label{fig:explain_cnn_vs_gnn}
\end{figure}

We identify a critical limitation in existing GNN explanation pipelines: \textbf{graph symmetrization} discards directional information, compromising explanation fidelity. 
Many GNNs and explainers assume undirected graph structures, as spectral-based models like GCN~\cite{kim2019edge} and ChebyNet~\cite{defferrard2016convolutional} require symmetric inputs (see Appendix~\ref{appendix:type} and \ref{appendix:relax} for a detailed discussion).

However, edge directionality is crucial in real-world tasks. In financial fraud detection, for instance, transaction flows are inherently directed, and symmetrization distorts their structure. Consequently, GNN explainers operating on such graphs fail to capture true causal dependencies, producing misleading interpretations.

Despite its widespread use, the impact of graph symmetrization on explanation quality remains understudied~\cite{yuan2022explainability}. We provide the first systematic theoretical analysis of this issue, quantifying the information loss from symmetric relaxation and demonstrating its effect on explanation quality. Empirical results on synthetic and real-world datasets show that preserving directional semantics  improves explanation quality, highlighting the need for direction-aware explainability in high-stakes applications.

\begin{figure*}
    \centering
    \includegraphics[width=\textwidth]{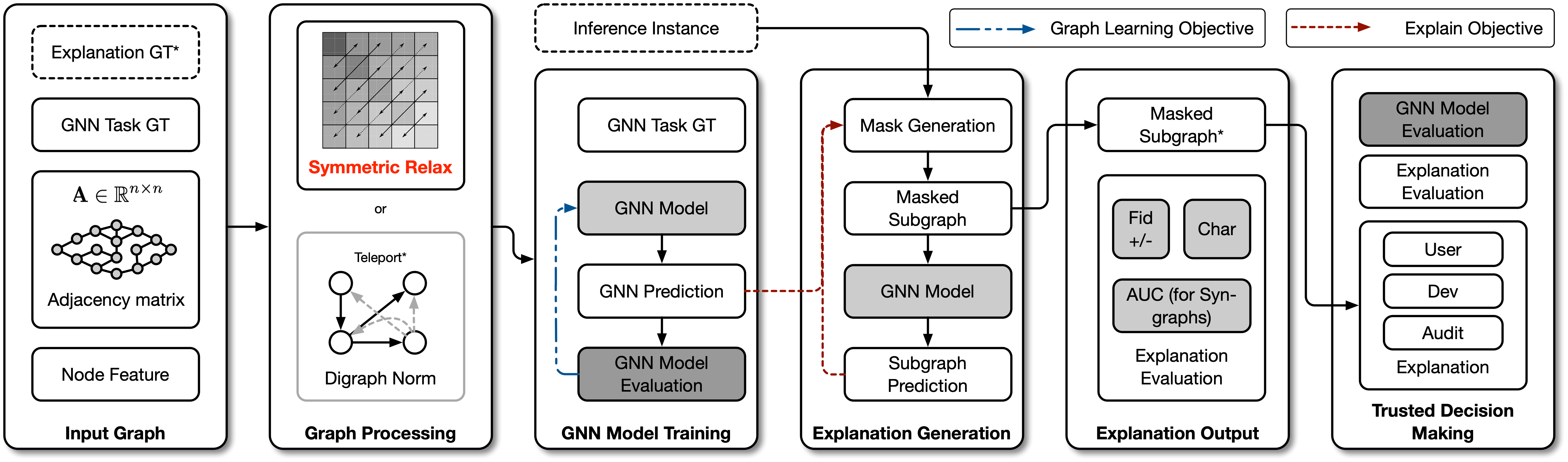}
    \caption{
        Overview of the explainable graph explanation pipeline. It consists of four key components: (1) Graph data and processing, (2) GNN model training, (3) Explanation generation and evaluation, and (4) Trusted Graph Learning. 
    }
    \label{fig:fw}
\end{figure*}

\section{Graph Learning Explanation}

A \textbf{graph} is represented as \( \mathcal{G} = (\mathcal{V}, \mathcal{E}) \), where \( \mathcal{V} \) denotes the set of nodes and \( \mathcal{E} \) represents directed or undirected edges. The graph structure is captured by its adjacency matrix \( \mathbf{A} \in \{0, 1\}^{n \times n} \), where \( n = |\mathcal{V}| \). Each node is associated with a feature matrix \( \mathbf{X} \in \mathbb{R}^{n \times c} \), where \( c \) is the number of features per node.

\textbf{Post-hoc explanation methods} aim to interpret trained GNNs black-boxes by identifying subgraph structures that drive model predictions. GNN explainers operate by selecting key subgraph components that most influence the model’s output.

\begin{definition}[GNN Explanation]
A \textit{GNN explanation} identifies a subgraph \( \mathcal{G}_{s} \) and its node features that are most influential in determining a model’s prediction. The objective is formulated as:
\begin{equation}
    \begin{array}{l}
        \max_{\mathcal{G}_{s} \subseteq \mathcal{G}} \textup{MI}( \mathbf{Y},\ (\mathcal{G}_{s} ,\mathbf{X}_{s}))\\
        \textup{s.t.}\ |\mathcal{V}_{s} |\leqslant k
    \end{array}
\end{equation}
where \( \mathcal{G}_{s}=(\mathcal{V}_s, \mathcal{E}_s) \) is a subgraph of \( \mathcal{G} \), \( \mathbf{X}_s \) denotes its node features, \( \mathbf{Y} \) represents the model’s output, and \( \textup{MI}(\cdot, \cdot) \) is the \textit{mutual information} between the explanation and the prediction ~\cite{ying2019gnnexplainer, yuan2022explainability}. The constraint size of subgraph \( k \) ensures that explanations remain concise and interpretable.
\end{definition}

\textbf{The explanation pipeline}, as illustrated in Figure~\ref{fig:fw}, consists of three key components that interact to generate human-understandable insights, enabling users to interpret, trust, debug, and audit GNN decisions.

\textit{Graph data processing} prepares raw graph data, including adjacency matrix \( \mathbf{A} \), node features \( \mathbf{X} \), and labels, for model training. A key preprocessing step determines how directional information is handled. The conventional \textit{symmetric relaxation} (Symm) approach converts directed graphs into undirected ones for compatibility with spectral GNNs like GCN and ChebyNet. In contrast, this work introduces \textit{Laplacian Normalization (Lap-Norm)}~\cite{tong2020digraph}, a novel method that preserves asymmetric structural relationships, enabling a comparative study of explainability in directed and undirected GCNs. The preprocessing choice directly affects explanation quality and interpretability for human analysts.

\textit{The base GNN model} is then trained on the processed graph, which serves as the foundation for explanation generation. 
\textit{The explainer }optimized on the frozen trained model highlights the most influential graph components by learning node and feature masks, isolating critical subgraphs, and assessing how these elements impact model predictions. This process directly supports human users—model developers gain insights into how decisions are made, auditors can evaluate the model’s compliance with expected behaviors, and domain experts can verify whether the explanations align with domain knowledge.

We assess explanation quality using multiple \textit{metrics}. \textit{Fidelity} (\text{Fid}) measures how well an explanation preserves the model’s original prediction. It consists of two components:
\begin{equation}
    \text{Fid}^+ = P(Y | \mathcal{G}_s) - P(Y | \mathcal{G}), 
    \text{Fid}^- = P(Y | \mathcal{G}) - P(Y | \mathcal{G} \setminus \mathcal{G}_s),
\end{equation}
where \( \mathcal{G} \) is the original graph, \( \mathcal{G}_s \) is the explanation subgraph, and \( P(Y | \mathcal{G}) \) denotes model confidence. Higher \(\text{Fid}^+\) and lower \(\text{Fid}^-\) indicate more faithful explanations.

\textit{Characterization Score} (\text{Char})~\cite{amara2022graphframex} balances \textit{sufficiency} and \textit{necessity}, ensuring that \( \mathcal{G}_s \) both supports the prediction and reflects critical structures. Defined as:
\begin{equation}
    \text{Char} = \frac{\left(w_+ + w_-\right) \cdot \text{Fid}^+ \cdot (1 - \text{Fid}^-)}{w_+ \cdot (1 - \text{Fid}^-) + w_- \cdot \text{Fid}^+},
\end{equation}
where \( w_+ + w_- = 1 \). Higher scores indicate better alignment with the model’s decision-making process.

\textit{The final output} consists of a subgraph and its associated node features that provide the most interpretable explanation of the GNN’s decision for a given instance. These explanations are presented to human stakeholders to enhance trust in the model, facilitate debugging when unexpected behavior arises, and support governance in high-stakes decision-making environments.

\section{Theoretical Analysis}

Many GNNs and explainers assume undirected graphs, leading to the common practice of \textit{graph symmetrization}. While this simplifies computations, it removes essential directional information, affecting both GNN prediction accuracy and post-hoc explanations. We formally analyze this \textit{information loss} and its consequences.

\subsection{Information Loss in Graph Symmetrization}

Entropy-based methods quantify structural complexity in graphs. The \textit{von Neumann entropy}~\cite{ye2014approximate} measures how much information is retained in a graph’s structure. 

\begin{definition}[von Neumann Entropy]
    Given adjacency matrix \( \mathbf{A} \) and Laplacian \( \mathbf{L} = \mathbf{D} - \mathbf{A} \), the von Neumann entropy of a graph is:
    \begin{equation}
        H_v(\mathcal{G}) = -\text{Tr}(\mathbf{L} \log \mathbf{L}).
    \end{equation}
\end{definition}

For a directed graph \( \mathcal{G} \) and its symmetrized version \( \mathcal{G}^u \), entropy satisfies:
\begin{equation}
    H_v(\mathcal{G}) \geqslant H_v(\mathcal{G}^u).
\end{equation}
This indicates that symmetrization reduces structural complexity, discarding meaningful asymmetric patterns.

\subsection{Information Loss for GNN Explainers}

Mutual information (MI) quantifies how much of a subgraph \( \mathcal{G}_s \) contributes to model predictions. Theorem~\ref{thro:mi_comp} formalizes the impact of symmetrization on MI.

\begin{definition}[Mutual Information]
    Given labels \( \mathbf{Y} \) and subgraph \( \mathcal{G}_s \subseteq \mathcal{G} \), mutual information is:
    \begin{equation}
        \textup{MI}(\mathbf{Y}, \mathcal{G}_s) = H(\mathbf{Y}) - H(\mathbf{Y} | \mathcal{G}_s).
    \end{equation}
\end{definition}

\begin{theorem}[Directional Semantic Gain]\label{thro:mi_comp}
    For a directed graph \( \mathcal{G} \) and its symmetrized version \( \mathcal{G}^u \),
    \begin{equation}\label{eq:inequal}
        \max_{\mathcal{G}_{s} \subseteq \mathcal{G}} \textup{MI}(\mathbf{Y}, \mathcal{G}_{s}) \geqslant 
        \max_{\mathcal{G}_{s}^{u} \subseteq \mathcal{G}^{u}} \textup{MI}(\mathbf{Y}, \mathcal{G}_{s}^u),
    \end{equation}
    where \( |\mathcal{G}_{s}^u| \leqslant k \) and \( |\mathcal{G}_{s}| \leqslant k \).
\end{theorem}

\begin{proof}
From the definition of MI:
\begin{equation}
    \textup{MI}(\mathbf{Y}, \mathcal{G}_{s}) - \textup{MI}(\mathbf{Y}, \mathcal{G}_{s}^{u}) = H(\mathbf{Y} | \mathcal{G}_{s}^{u}) - H(\mathbf{Y} | \mathcal{G}_{s}).
\end{equation}
Since symmetrization function \( f^u \) ensures \( \text{Pr}(\mathcal{G}_{s}^{u} | \mathcal{G}_{s}) = f(\mathcal{G}_{s}) \), it follows that:
\begin{equation}
    H(\mathcal{G}_{s}^{u} | \mathcal{G}_{s}) = 0.
\end{equation}
Applying the conditional entropy inequality,
\begin{equation}
    H(\mathcal{G}_{s}^{u} | \mathbf{Y},\mathcal{G}_{s}) \leqslant H(\mathcal{G}_{s}^{u} | \mathcal{G}_{s}),
\end{equation}
we obtain:
\begin{equation}
    H(\mathbf{Y} | \mathcal{G}_{s}) - H(\mathbf{Y} | \mathcal{G}_{s}^u) \geqslant 0.
\end{equation}
Thus, \( \textup{MI}(\mathbf{Y}, \mathcal{G}_{s}) \geqslant \textup{MI}(\mathbf{Y}, \mathcal{G}_{s}^{u}) \), proving \eqref{eq:inequal}.
\end{proof}

\textbf{Implications for GNNs and Explainability.} 
~These results show that symmetrization negatively affects both GNN predictions and explanations. Removing directional information weakens the model’s ability to capture asymmetric dependencies, reducing predictive accuracy. Furthermore, explainers operating on symmetrized graphs generate misleading feature attributions, compromising interpretability. Our findings highlight the necessity of preserving directionality for more faithful and reliable GNN explanations.

\begin{table*}[ht]
    \begin{center}
    \caption{Synthetic graph datasets and their experiment results.}
    \begin{tabular}{
        >{\arraybackslash}m{2.2cm} 
            >{\centering\arraybackslash}m{2cm} 
            >{\centering\arraybackslash}m{2cm} 
            >{\centering\arraybackslash}m{2cm} 
            >{\centering\arraybackslash}m{2cm} 
            >{\centering\arraybackslash}m{2cm}
            >{\centering\arraybackslash}m{2cm}
    }
    \toprule[1px]
    \textbf{Dataset}& BA-Shapes & BA-Com & Tree-Cycles & Tree-Grid & \textbf{DiLink-Motif} & \textbf{DiLink-Base}\\
    \midrule 
        \multirow{2}{*}[1ex]{\textbf{Graph Base}} &
            \includegraphics[width=2cm, height=2cm]{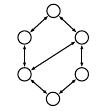}&
            \multirow{2}{*}[6ex]{\includegraphics[width=2cm, height=4.2cm]{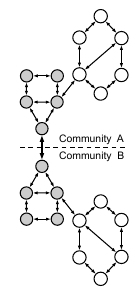}}&
            \includegraphics[width=2cm, height=2cm]{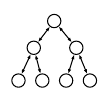}&
            \includegraphics[width=2cm, height=2cm]{figs/motifs/tree-base.pdf}&
            \includegraphics[width=2cm, height=2cm]{figs/motifs/bashape-base.pdf}&
            \multirow{2}{*}[6ex]{\includegraphics[width=2cm, height=4.2cm]{figs/motifs/bacom-basenmotif.pdf}} \\
        \multirow{2}{*}[1ex]{\textbf{Graph Motif}} &
            \includegraphics[width=2cm, height=2cm]{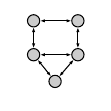}&
            
            &
            \includegraphics[width=2cm, height=2cm]{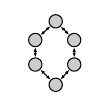}&
            \includegraphics[width=2cm, height=2cm]{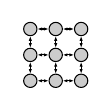}&
            \includegraphics[width=2cm, height=2cm]{figs/motifs/bashape-motif.pdf}&
             \\
         \multirow{2}{*}[1ex]{\textbf{Bonding Type}} &
            \includegraphics[width=2cm, height=2cm]{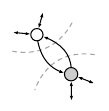}&
            \includegraphics[width=2cm, height=2cm]{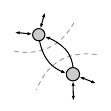}&
            \includegraphics[width=2cm, height=2cm]{figs/motifs/bashape-bond.pdf}&
            \includegraphics[width=2cm, height=2cm]{figs/motifs/bashape-bond.pdf}&
            \includegraphics[width=2cm, height=2cm]{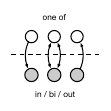}&
            \includegraphics[width=2cm, height=2cm]{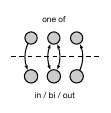} \\
    \textbf{Node Features} & None & $\mathcal{N}(\mu_l, \sigma_l)$ & None & None & None & None\\ [1ex]
    \textbf{Content} & Graph Structure & Graph / Node Feature & Graph Structure & Graph Structure & Graph Structure & Graph Structure\\ [1ex]
    \midrule[1px]
    \multicolumn{7}{c}{\normalfont{\textbf{Explanation AUC}}} \\
    \midrule
        GNNExplainer(B) &0.925 &0.836 &0.948 &0.875 &0.793 &0.777 \\
        GNNExplainer(L) &0.926 &0.833 &0.937 &0.900 &\textbf{0.914} &\textbf{0.905} \\
        PGExplainer(B) &0.963 &0.945 &0.987 &0.907 &0.778 &0.769 \\
        PGExplainer(L) &0.962 &0.943 &0.987 &0.910 &\textbf{0.905} &\textbf{0.895} \\
    \midrule
    \multicolumn{7}{c}{*(B): Bidirectional Symmetrical Relaxation; (L): Laplacian Normalization} \\
    \bottomrule[1px]
    \end{tabular}
    \end{center}
    \label{tab:main_exp_syn}
\end{table*}

\section{Empirical Study}

Our experiments evaluate whether the \textit{symmetric relaxation} process compromises explanation quality and identify the primary causes. We assess synthetic and real-world graph datasets, including our newly constructed DiLink dataset, using multiple GNN models and explainers.

\textit{Implementation.} We train GCN-based models following \cite{ying2019gnnexplainer, luo2020parameterized}, selecting the best models for explanation. Configuration files ensure reproducibility.\footnote{See supplementary materials for details.}

\textit{Synthetic Graphs.} We generate controlled datasets to evaluate explainers with ground-truth explanations. Following \cite{ying2019gnnexplainer, luo2020parameterized}, we construct Barabási-Albert (BA) graphs~\cite{albert2002statistical}, including BA-Shapes and BA-Community. To analyze directed graphs, we introduce \textit{DiLink}, which connects two base graphs via unidirectional or bidirectional edges, allowing systematic evaluation of symmetric relaxation. The advantages of using these synthetic graphs are that we could have access to the explanation ground truth label, and make it easier to analysis the cause of abnormal behavior.

\textbf{Key Findings.} Table~\ref{tab:main_exp_syn} compares GNNExplainer~\cite{ying2019gnnexplainer} and PGExplainer~\cite{luo2020parameterized} across synthetic datasets, evaluated using symmetric relaxation~\cite{kipf2016semi} and our proposed \textit{Laplacian normalization}. Results reveal that:

\textit{(1) Preserving directionality improves explanation quality.} On graphs where edge directionality is critical (e.g., DiLink variants), explanations generated with Laplacian normalization significantly outperform those using symmetric relaxation. For example, GNNExplainer’s AUC improves from 0.793 to 0.914 on DiLink-Motif.

\textit{(2) Direction-preserving preprocessing maintains compatibility with undirected graphs.} On traditional undirected benchmarks like BA-Shapes and Tree-Cycles, Laplacian normalization performs comparably to symmetric relaxation, showing that retaining directional information does not degrade explanation quality where directionality is irrelevant.

\textit{Real-World Graphs.} We also evaluate explainers on citation networks (Cora, Citeseer, PubMed) using preprocessing from \cite{tong2020digraph, tong2021directed}, as well as Amazon-Photo and Amazon-Computer datasets representing social networks with Fidelity metrics. Experiments on real-world datasets, confirm these trends~(Appendix \ref{appendix:real}).

\section{Discussion and Conclusion}

Graph Neural Networks (GNNs) are widely applied in fields such as criminal analysis \cite{li2024anomaly}, recommendation systems \cite{wang2019neural}, and drug discovery \cite{jiang2021could}. However, a critical limitation in existing GNN explanation pipelines is the loss of directional information caused by symmetric relaxation. This preprocessing step, while simplifying computations, compromises explanation fidelity by distorting the underlying graph structure, particularly in tasks where edge directionality encodes causal or temporal relationships.

\begin{figure}[b]
    \centering
    \includegraphics[width=0.6\linewidth]{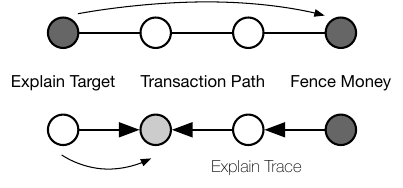}
    \caption{Motivating example: Tracing illicit money flows. Directional information enables accurate identification of transaction paths, whereas ignoring directionality risks generating misleading explanations.}
    \label{fig:example_fraud}
\end{figure}

In the end, to remind the readers the importance of preserving directional semantics again, consider a financial fraud detection scenario, \textbf{as shown in Figure~\ref{fig:example_fraud}}. The goal is to trace illicit money flows, starting from the \textit{Explain Target} (e.g., a flagged account) and following the \textit{Transaction Path} to the \textit{Fence Money} destination. Symmetrizing the graph removes the ability to distinguish transaction flows, generating misleading explanations that may lead investigators to the wrong conclusions. By retaining directional information, our framework accurately reconstructs the \textit{Explain Trace}, enabling actionable insights for detecting fraudulent activities. This example highlights how direction-aware explanations are essential for reliable analysis in high-stakes applications.

Despite its potential, graph explanation faces unique challenges. Unlike images or text, graphs often lack direct human-interpretable semantics, making explanations harder to understand. Current methods primarily support model debugging—similar to program repair—by identifying issues in learning pipelines. \textbf{Future research} should explore ways to make graph explanations more accessible to humans, possibly through domain-specific visualization tools or symbolic reasoning. Additionally, frameworks must be generalized to handle diverse graph types, such as hypergraphs and heterogeneous networks, ensuring robust explainability across a wide range of applications.

While explainability enhances trust in GNNs, it also introduces risks such as exposing sensitive patterns or enabling model extraction \cite{lords2018ai}. Privacy-preserving techniques, such as differential privacy \cite{arora2019differentially}, should be integrated into future pipelines to address these concerns and ensure responsible deployment of explainable GNN systems.

In conclusion, this work identifies a key flaw in GNN explanation pipelines: the loss of directional information through symmetric relaxation. Our theoretical analysis quantifies the resulting information loss, and our empirical results demonstrate that preserving directional semantics significantly improves explanation fidelity and reveals critical structural patterns. Moving forward, research should focus on creating human-centric and privacy-preserving explanation frameworks while extending their applicability to diverse graph structures. By addressing these challenges, we can foster robust, ethical, and interpretable GNN explanations for real-world applications.

{
    \small
    \bibliographystyle{unsrt}
    \bibliography{myref}
}

\appendix

\subsection{Different Types of Graph Neural Networks.} 
\label{appendix:type}
Modern graph neural networks primarily follow spectral or spatial paradigms with varying directional awareness. 

\textbf{Spectral approaches} GNNs define convolutions in the frequency domain using the graph Laplacian’s eigen-decomposition~\cite{bruna2013spectral, kipf2016semi}. Models such as GCN~\cite{kipf2016semi} and ChebyNet~\cite{defferrard2016convolutional} rely on the normalized Laplacian \( \mathbf{L} = \mathbf{I} - \mathbf{D}^{-1/2} \mathbf{A} \mathbf{D}^{-1/2} \), which is symmetric in undirected graphs, ensuring stable spectral filtering. However, in directed graphs, \( \mathbf{L} \) becomes asymmetric, leading to complex eigenvalues that destabilize convolutional operations. To circumvent this, spectral GNNs enforce \textit{graph symmetrization} by defining \( \mathbf{A}^u = \mathbf{A} + \mathbf{A}^\top \), inherently discarding directional information.Subsequent spectral variants address this limitation: DiGCN \cite{tong2020digraph} introduces directed Laplacians, while MagNet \cite{zhang2021magnet} employs magnetic Laplacians with complex potentials to preserve edge directions, and Dir-GNN \cite{rossi2024edge} incorporates Dirichlet energy constraints for digraphs.

\textbf{Spatial approaches} demonstrate inherent directional potential through localized aggregation. GraphSAGE \cite{hamilton2017inductive} supports directed neighborhoods through asymmetric sampling but often evaluates on symmetrized benchmarks. GAT \cite{velivckovic2017graph} preserves directionality via attention mechanisms yet predominantly tests on undirected graphs. While GIN \cite{xu2018powerful} theoretically handles directionality through injective aggregation, its isomorphism focus favors undirected implementations.

This progression highlights an important trade-off: traditional spectral methods often compromise on capturing directionality for the sake of mathematical simplicity, whereas modern spectral techniques and spatial approaches are increasingly designed to preserve directional information, which is vital for real-world graph learning tasks. However, it is worth noting that while spatial-based GNNs are theoretically capable of handling directed graphs, these approaches have not been thoroughly explored or evaluated on directed graph datasets.

\subsection{Graph Symmetric Relaxation}
\label{appendix:relax}
\begin{figure}[htbp]
    \centering
    \includegraphics[width=\columnwidth]{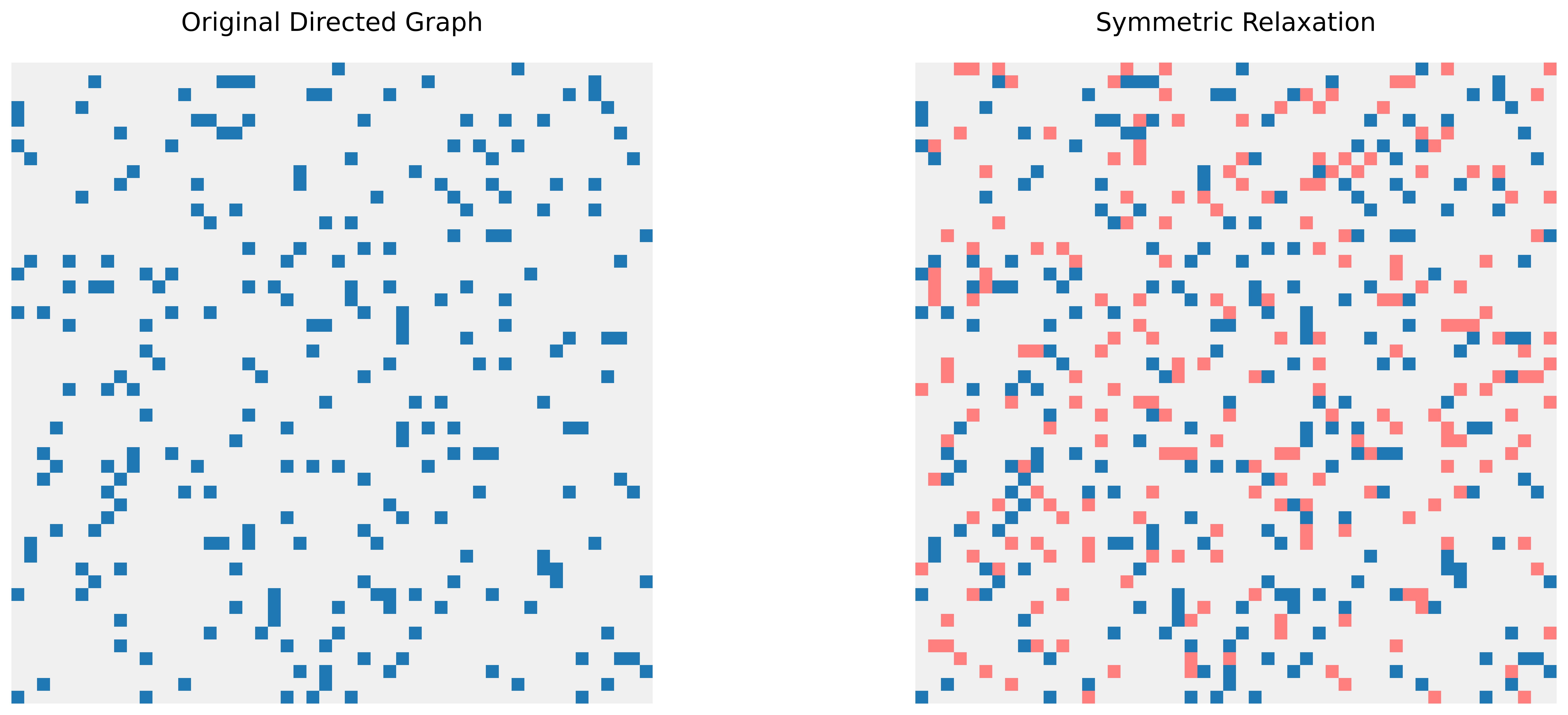}
    \caption{Comparison of adjacency matrices before and after symmetric relaxation preprocessing. 
    Left: The original directed graph adjacency matrix, where blue entries indicate directed edges from row nodes to column nodes. 
    Right: The symmetrically relaxed adjacency matrix, where blue entries represent the original directed edges, and red entries highlight the newly added reverse edges introduced during the symmetric relaxation process. 
    The symmetric relaxation operation ensures that for any edge $(i,j)$ in the graph, a corresponding edge $(j,i)$ is added, resulting in an undirected graph representation. 
    This visualization demonstrates how the symmetric relaxation transforms a directed graph structure while preserving the original connectivity patterns.}
    \label{fig:adj_matrices}
\end{figure}

The adjacency matrix modification process is illustrated in Figure~\ref{fig:adj_matrices}. 
As shown in the visualization, the symmetric relaxation preprocessing step transforms the original directed graph structure into an undirected representation by adding reverse edges. 
This transformation maintains the connectivity information of the original graph while ensuring symmetry in the adjacency matrix, which is often required for certain graph learning algorithms. 
The color-coding scheme (blue for original edges, red for added reverse edges) clearly demonstrates the structural changes introduced by this preprocessing step, allowing for visual assessment of the transformation's impact on the graph topology.

\subsection{Experiments on Real-world Datasets}\label{appendix:real}
GNNExplainer higher fidelity and directional consistency with direction preserving graph processing method in those GNNs. Additionally, it extends these findings across multiple GNN architectures, highlighting the broad applicability of our approach.

\begin{table}[t]
    \begin{center}
    \caption{Main results on \textbf{real-world graph datasets} 
    }
    \resizebox{\linewidth}{!}{
        \begin{tabular}{cccccccc}
        \toprule
            \textbf{Metric} & \textbf{Method} & Cora & CiteSeer & PubMed & Am-C & Am-P \\
        \midrule
            \multirow{4}{*}{$\uparrow$Fid+} 
                &G+B &0.673 &\textbf{0.885} &0.743 &0.982 &0.950 \\
                &G+L &\textbf{0.824} &0.867 &\textbf{0.837} &\textbf{0.989} &\textbf{0.957} \\
            \addlinespace[\belowrulesep] \cdashline{3-7}[1pt/2pt] \addlinespace[\aboverulesep]
                &P+B &0.291 &0.251 &0.770 &\textbf{0.468} &\textbf{0.531} \\
                &P+L &\textbf{0.914} &\textbf{0.914} &\textbf{0.819} &0.467 &\textbf{0.531} \\
        \midrule
            \multirow{4}{*}{$\downarrow$Fid-} &G+B &0.853 &0.885 &0.560 &0.480 &0.610 \\
                &G+L &\textbf{0.329} &\textbf{0.867} &\textbf{0.317} &\textbf{0.426} &\textbf{0.577} \\
            \addlinespace[\belowrulesep] \cdashline{3-7}[1pt/2pt] \addlinespace[\aboverulesep]
                &P+B &0.567 &\textbf{0.567} &0.440 &\textbf{0.227} &\textbf{0.402} \\
                &P+L &\textbf{0.439} &0.606 &\textbf{0.435} &0.364 &0.448 \\
        \midrule
            \multirow{4}{*}{$\uparrow$Char}&G+B &0.242 &0.560 &0.552 &0.680 &0.553 \\
                &G+L &\textbf{0.748} &\textbf{0.827} &\textbf{0.752} &\textbf{0.727} &\textbf{0.587} \\
            \addlinespace[\belowrulesep] \cdashline{3-7}[1pt/2pt] \addlinespace[\aboverulesep]
                &P+B &0.384 &0.307 &0.652 &0.539 &0.541 \\
                &P+L &\textbf{0.588} &\textbf{0.588} &\textbf{0.665} &\textbf{0.701} &\textbf{0.716} \\
    \bottomrule[1px]
    \end{tabular}
    }
    \begin{tablenotes}
        \item Performance evaluation of the two target explainers with different edge processing pipeline. Method: \textbf{G} means GNNExplainer, \textbf{P} means PGExplaienr, \textbf{B} and \textbf{L} are two different graph preprocessing methods.
        All datasets are made for GNN base models to train and inference node classification task.
        We use metrics to evaluate the explanation quality: Fidelity (Fid+/-), and Characteristic Score (Char). 
        The experiment results show that the explanation quality are better when directional information are preserved in the explanation pipeline.
    \end{tablenotes}
    \end{center}
    \label{tab:main_exp_real}
\end{table}

\subsection{Learning on Directed Graph Structured Data} \label{sec:related-dgnn}
Graph Neural Networks, first introduced by \cite{gori2005new} and \cite{scarselli2008graph}, have achieved remarkable success in various applications. Notable works include Graph Convolution Network (GCN) \cite{kipf2016semi}, GraphSAGE \cite{hamilton2017inductive}, and Graph Attention Network (GAT) \cite{velivckovic2017graph}. \cite{defferrard2016convolutional}'s approach enabled graph classification capabilities, demonstrating GNNs' ability to capture complex structural information in graph-structured data.

These approaches generally fall into two categories: spectral-based and spatial-based methods. Spectral-based GNNs like GCN rely on graph Laplacian matrices to define convolution operations, while spatial-based methods like GAT and Graph Isomorphism Network (GIN) utilize local neighborhood information for feature aggregation. The spatial methods typically focus on outgoing neighbors for computational efficiency, often at the cost of directional information.

For directed graphs, spectral methods face inherent challenges as they require symmetric, positive semi-definite matrices for spectral decomposition. Early solutions typically symmetrized the adjacency matrix, losing directional information in the process. Recent works have proposed various solutions: \cite{ma2019spectral} utilized random walk matrices and their stationary distributions, while \cite{tong2020digraph} introduced directed graph Laplacian with PageRank and auxiliary nodes to preserve directional information. Further innovations include MagNet \cite{zhang2021magnet}, which employs complex-valued adjacency matrices to maintain graph directionality.

\subsection{Explainability of GNNs: an Overall Review} \label{sec:related-explain}
As GNNs become increasingly prevalent in critical applications, explaining their predictions has emerged as a crucial challenge. Post-hoc explanation methods \cite{adadi2018peeking,fisher2019all,guidotti2018survey,koh2017understanding} have been widely adopted, viewing models as black boxes while probing for relevant information. In computer vision, techniques like Grad-CAM \cite{selvaraju2017grad} use gradient information to highlight important input regions.

In the context of GNNs, explanation methods can be categorized into instance-level and model-level approaches. GNNExplainer \cite{ying2019gnnexplainer} pioneered instance-level explanations by identifying important nodes and edges through masking, while \cite{luo2020parameterized} introduced a parameterized approach using neural networks for explanation generation.

Instance-level methods include gradient-based techniques that leverage model gradients for feature importance, decomposition-based approaches that break predictions into interpretable substructures, and perturbation-based methods that explore the impact of graph modifications. Model-level approaches aim to provide global understanding through techniques like surrogate models and counterfactual explanations. However, most existing methods overlook the importance of edge directionality in their explanations, potentially missing crucial structural information in directed graphs that could improve both local and global interpretability.

\subsection{Discussion: Bridging Transparency and Decision Justification in AI Systems}\label{appendix:explain}

\begin{quote}
 . . . One solution to the question of intelligibility is to try to increase the technical transparency of the system, so that experts can understand how an AI system has been put together. This might, for example, entail being able to access the source code of an AI system. However, this will not necessarily reveal why a particular system made a particular decision in a given situation . . . An alternative approach is explainability, whereby AI systems are developed in such a way that they can explain the information and logic used to arrive at their decisions. --- \textit{The United Kingdom House of Lords}. \cite{lords2018ai} 
\end{quote}

As emphasized by the UK House of Lords (see Appendix \ref{appendix:explain}), merely accessing a system's technical implementation does not necessarily reveal its decision-making logic. This limitation is particularly pronounced in GNNs, where the interplay between graph structure and node features creates complex patterns that are not immediately apparent from the model architecture alone. The need for explainability becomes even more critical when GNNs are deployed in safety-critical scenarios, where decisions can have significant real-world impacts \cite{leslie2019understanding}.

For instance, in financial systems, a GNN might flag a transaction as fraudulent based on complex patterns in the transaction network. Without proper explanation capabilities, it becomes challenging to validate the model's decisions, ensure fairness, and maintain regulatory compliance \cite{weber2019anti}. This underscores why developing robust explanation frameworks for GNNs is not just a technical challenge, but a crucial requirement for their responsible deployment in practice.

\end{document}